\documentclass[manyauthors]{fundam}
\usepackage{hyperref}
\usepackage{manyfoot}
\usepackage{graphicx}
\usepackage{theorem}
\usepackage{color}
\usepackage{amssymb,amsmath,graphicx,tikz}
\usepackage{multirow}

\newcommand{\N}{\mathbb{N}}
\newcommand{\C}{\mathcal{C}}

\begin{document}

\title{Normal forms in Virus Machines\thanks{Supported by the Zhejiang Lab BioBit Program (Grant No. 2022BCF05).}}

\author{Antonio Ram\'irez-de-Arellano\thanks{Corresponding author: Marie Sklodowska-Curie Actions grant 101206770 - VM-AI.}\\
    Research Institute for Artificial Intelligence\\
    "Mihai Draganescu", Bucharest 040711, Romania\\
    SCORE Laboratory, I3US, University of Seville\\
    antonio@racai.ro aramirezdearellano@us.es
    \and Francis George C. Cabarle\thanks{Supported by Consejer{\'i}a de Transformaci{\'o}n Econ{\'o}mica, Industria, Conocimiento y Universidades, de la Junta de Andaluc{\'i}a, in the framework of the project ``Smart Computer systems Research and Engineering (SCORE)'' corresponding to Programa de Fortalecimiento de Institutos Universitarios de Investigaci{\'o}n de las Universidades Andaluzas, Centros e Infraestructuras para la adquisici{\'o}n del sello ``Severo Ochoa'' o ``María de Maeztu''.}\\
    SCORE Laboratory, I3US, University of Seville\\
    Dept. of Computer Science\\
    University of the Philippines Diliman\\
    fcabarle@us.es fccabarle@up.edu.ph
    \and
    David Orellana-Mart\'in\\
    Dept. of Computer Science and Artificial Intelligence\\
    University of Seville\\
    SCORE Laboratory, I3US, University of Seville\\
    dorellana@us.es
    \and
    Mario J. P\'erez-Jim\'enez\\
    Dept. of Computer Science and Artificial Intelligence\\
    University of Seville\\
    SCORE Laboratory, I3US, University of Seville\\
    marper@us.es
}

\maketitle

\runninghead{A. Ram\'irez-de-Arellano, F. G. C. Cabarle, D. Orellana-Mart\'in, M. J. P\'erez-Jim\'enez}{Normal Forms in VM}

\abstract{
In the present work, we further study the computational power of virus machines (VMs in short).
VMs provide a computing paradigm inspired by the transmission and replication networks of viruses.
VMs consist of process units (called hosts) structured by a directed graph whose arcs are called channels and an instruction graph that controls the transmissions of virus objects among hosts.
The present work complements our understanding of the computing power of VMs by introducing normal forms; these expressions restrict the features in a given computing model.
Some of the features that we restrict in our normal forms include (a) the number of hosts, (b) the number of instructions, and (c) the number of virus objects in each host.
After we recall some known results on the computing power of VMs we give our series of normal forms, such as the size of the loops in the network, proving new characterisations for families of number sets, such as finite sets, semilinear sets, or recursively enumerable sets ($NRE$). 
}



\maketitle

\section{Introduction}

In the present work, we consider some normal forms for virus machines~\cite{VM} (VMs), this computing paradigm develops unconventional and natural computing models inspired by networks of virus replications and transmissions.
More information on unconventional and natural computing is found in \cite{unconcomp_handb2021} and \cite{naco_handbook2012}, respectively.
From \cite{VM,GCRVM} it is shown that VMs are {\it Turing complete}, that is, they are algorithms capable of general-purpose computations.
From such works some VMs for computing classes of (in)finite sets of numbers are also shown.
Providing normal forms for VMs allows a more refined or deeper view of their computations: What features and their values for VMs can be increased or decreased to increase or decrease computing power?

Virus machines consists of three {\it graphs}: a directed and weighted {\it host graph} with nodes and edges referred to as {\it hosts} and {\it channels}, respectively;
a directed and weighted {\it instruction graph} where nodes are {\it instructions} and edge weights determine which instruction to prioritise and next activate;
an {\it instruction-channel} graph which connects an edge between instructions and channels in the previous graphs.
Hosts contain zero or more virus objects, and activating an instruction means opening a channel since the channels are closed by default.
Opening a channel means that virus objects from one host are replicated and transferred to another host.
Even from this high-level view of VMs and their features, it can be evident that the source of their computing power can be further restricted or refined.
Thus, a deeper understanding of VMs as algorithms is gained in order to aid in developing applications.

Briefly, the idea of a normal form for some computing mo\-del is to consider restrictions in the model while maintaining its computer power.
That is, considering lower bounds for ingredients in a computing model is a natural direction for investigation.
For instance, a well-known normal form in language theory is the Chomsky normal form, CNF in short, from \cite{chomsky_formgram1959}.
Instead of having an infinite number of forms to write rules in a grammar for context-free sets, CNF shows that two forms are enough.
Normal forms in unconventional or bio-inspired models include spiking neural P systems \cite{ibarra_snpnorm2007} and cellular automata \cite{brownca2010}, with recent and optimal results in \cite{ivan_snpvarnorm2022,peper_normbca2012}, and a recent survey in \cite{francis_acmc2022jour}.
In addition to restriction while maintaining the same power, normal forms can provide {\it frontiers}.
For instance, by giving some lower bound for some value, further decreasing the value can mean that we only compute a proper subset of problems as before.

This paper contributes the following to the study of virus machines and their computing power.
Normal forms, some of which are optimal bounds, for VMs are provided in the following sense: (a) providing characterisations (previously were inclusions) for generating families of finite sets;
(b) showing new characterisations for finite sets of numbers using restric\-tions on the number of required hosts, instructions, or viruses; and (c)
new characterisations are also given for singleton sets of numbers and some linear progressions or combinations.
The previous work was focused on the restriction of only three ingredients: hosts, instructions, and number of viruses. In this work, new restrictions are considered: limiting or not the host or ins\-truction graphs to be a tree graph, that is, an acyclic directed graph. Moreover, the instruction-channel graph can also be li\-mited in the sense that one channel can be attached to at most one instruction.
We show, for instance, that some VMs with a tree instruction or host graph and with some lower bounds on the number of hosts, instructions, and viruses can only compute finite sets. Lastly, we highlight a new (and better) characteri\-sation of semilinear sets of numbers, in short $SLIN$, with virus machines.
Our results on normal forms are then used to ask new questions regarding other normal forms and restrictions on VMs.

The present work is a much extended and revised version of the preliminary report in~\cite{tree_or_not_bwmc2024}.
For example, findings concerning the family of sets $SLIN$ are particularly noteworthy from both a theoretical and an applied perspective:
$SLIN$ is a  class enjoying the fact that it is above the family of finite sets $NFIN$ and below the family of Turing computable sets $NRE,$ with $SLIN$ known to be decidable \cite{ginsspan1966}; 
the decidability of $SLIN$ helps toward the computational complexity of machines for it;
applications can benefit from the decidability of $SLIN$, including formal verification, proof assistant \cite{nipkow2008linear}.

The organisation of the present work is as follows. First, some brief definitions and the state of the art are presented in Section~\ref{sec:def}. After that, novel results are presented with the old ingredients in Section~\ref{sec:oldIng}. We continue with novel results with the new ingredients proposed in Section~\ref{sec:newIng}. Lastly, some conclusions with open remarks are shown in Section~\ref{sec:conc}.

\section{Definitions}\label{sec:def} 

\subsection{Virus Machines}\label{def:vm}
\begin{definition}\label{def:VM}
    Let a virus machine $\Pi$ of degree $(p,q)$ with $p,q \geq 0$ the number of hosts and instructions respectively, defined as:

    $$\Pi = (\Gamma, H, I, D_H, D_I, G_C, n_1,\dots,n_p,i_1,h_{out})$$

    where:

    \begin{itemize}
        \item $\Gamma=\{v\}$ is the singleton alphabet.
        \item $H=\{h_1,\dots,h_p\}$ is the ordered set of hosts, $h_{out}$ can be either in $H$ or not (for this work, we will suppose always $h_{out}= h_0\notin H$, $I=\{i_1,\dots, i_q\}$ the ordered set of instructions).
        \item $D_H = (H\cup \{h_{out}\}, E_H, w_H)$ is the weighted and directed (WD) \textit{host graph}, where the edges are \textit{called channels} and $w_H : E_H\cup\{h_{out}\} \rightarrow \N$.
        \item $D_I = (I, E_I, w_I)$ is the WD \textit{instruction graph} and $w_I : E_I\rightarrow \{1,2\}$.
        \item $G_C = (E_H \cup I, E_I)$ is an unweighted bipartite graph ca\-lled \textit{channel-instruction} graph, where the disjoint partition associated is $\{E_H\cup I\}$, although a channel can be attached to more than one instruction, each instruction is attached to only one channel at most.
        \item $n_1,\dots,n_p\in \N$ are the initial number of viruses in each host $h_1,\dots, h_p$, respectively.
    \end{itemize}
\end{definition}

Regarding the \textbf{semantics}, a \textit{configuration} at an instant $t\geq0$ is the tuple $\C_t = (a_{1,t},a_{2,t},\ldots,a_{p,t},u_t,a_{0,t})$
where for each $j\in\{1,\dots, p\}$, $a_{j,t}\in\N$ represents the number of viruses in the host $h_j$ at instant $t$, and $u_t\in I\cup\{\#\}$ is the following activated instruction, unless $u_t= \#$ that is a halting configuration. Lastly, $\C_0=(n_1,\dots,n_p,i_1,0)$ is the \textit{initial configuration}. 

From a configuration $\C_t$, $\C_{t+1}$ is obtained as follows. The instruction that will be activated is $u_t$ if $u_t\in I$, otherwise $\C_t$ is a \textit{halting configuration}. Let us suppose that $u_t\in I$ and that it is attached to the channel $(h_j,h_{j'})\in E_H$ with weight $w_H((h_j,h_{j'}))=w\in\N$, then the channel is \textit{opened} and  two possibilities holds:
\begin{itemize}
    \item If $a_{j,t}>0$, then there is \textit{virus transmission}, that is, one virus is consumed from $h_j$ and is sent to the host $h_{j'}$ replicated by $w$. The next activated instruction will follow the highest weight path from $u_t$ to just the next instruction on the instruction graph. In case the highest path is not unique, it is chosen nondeterministically. In case there is no possible path, then $u_{t+1} = \#$
    \item If $a_{j,t}=0$, then there is no virus transmission and the next instruction follows the least weight path from $u_t$ to just the next instruction. For the other cases, it is analogous to the previous assumption.
\end{itemize}

This paper is focused on the computational power of VMs in the generating mode, that is the numbers that can be generated in all the possible computations, and for that we fix the same notation as in the foundational paper~\cite{computing-VM} which will be called {\it old ingredients}. Let
$\mathrm{NVM}(p,q,n)$ be the family of sets of natural numbers generated by virus machines with at most $p$ hosts, $q$ instructions, and $n$ viruses in each host at any instant of the computation. For unbounded restrictions, they are replaced by $*$.

\subsubsection{Explicit Example}

To clarify the behavior of these devices, the generating mode, and their visual representation, let us see an explicit example. Let $\Pi_{ex}$ be the VM of degree $(2,4)$ in the generating mode and visually described in Figure~\ref{fig:VM-ex}. The squares represent the host units with $h_0$ being the environment, and the double arrows the arcs of the host graph. For simplicity, if the weight is $1$, then the number is omitted in the representation. The blue dots are the instructions units and the vertices in the instruction graph. Lastly, the red dashed lines are the edges of the instruction-channel graph.

\begin{figure}[ht]
    \centering
    \includegraphics[width=0.4\textwidth]{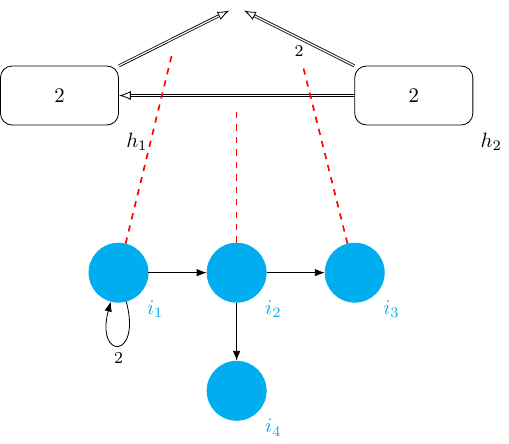}
    \caption{The virus machine $\Pi_{ex}$.}
    \label{fig:VM-ex}
\end{figure}

The computation works as follows; the initial configuration is $\C_0=(2,2,i_1,0)$. The instruction $i_1$ is connected to the channel $(h_1,h_0)$, as the host $h_1$ is not empty, then there is virus transmission and the highest weight path is followed, that is, it goes again to $i_1$ with configuration $\C_1=(1,2,i_1,1)$. The process is analogous and the configuration $\C_2=(0,2,i_1,2)$ is reached. From here, since the host $h_1$ is empty, there is no virus transmission and the least weight path is followed, having the configuration $\C_3=(0,2,i_2,2)$. Instruction $i_2$ opens the channel $(h_2,h_1)$ so there is virus transmission, but now the highest weight path is not unique, having two possible configurations depending on the instruction that is chosen:
\begin{itemize}
    \item If instruction $i_3$ is chosen, then the configuration is $\C_4=(1,1,i_3,2)$. That instruction is connected to the channel $(h_2,h_0)$ with weight $2$, thus the virus sent from host $h_2$ is replicated by $2$ and sent to the output region. As there is no possible path from the instruction graph, a halting configuration is reached $\C_5=(1,1,\#,4)$. This means that the number $4$ has been generated.
    \item  If instruction $i_4$ is chosen, then the configuration is $\C_4=(1,1,i_4,2)$. As that instruction is not attached to any channel and there is no possible path in the instruction graph, a halting configuration is reached $\C_5=(1,1,\#,2)$. This means that the number $2$ has been generated.
\end{itemize}

With all the possible computations checked, we can state that the VM $\Pi_{ex}$ generates the set $\{2,4\}$.

\subsubsection{Formal Verification}

One technique to mathematically verify the integrity of these devices is by designing invariant formulas which highlight interesting properties of the most relevant loops. For simplicity, some formal verifications will show the invariant formula and how it works; the mathematical proof of these formulas can be easily obtained with the induction technique. Lastly, we propose a novel technique for formal verification, that is, by fixing the sizes of the loops in the instruction graph. 

\subsection{Topology}
For this study, several notation and clarifications related to discrete topology must first be presented.

\begin{definition}
    A path in a directed graph $G = (V,E)$ is an ordered tuple\\
    $(v_1,\dots,v_n)$ of vertices such that $(v_i,v_{i+1}) \in E$ for $i=1,\dots, n-1$ and $v_i\neq v_j,$ for $i,j= 1,\dots,n,$ and $i\neq j$. Under the same conditions, if $(v_n,v_1)\in E$, then the path is called a cycle. A graph without cycles is called a tree. The \textit{depth} of a tree is the longest path of the tree.

    We say that $v_1$ is connected or attached to $v_n$ if there is a path $w= (e_1,\dots,e_{n-1})$, whose sequence of vertices is $(v_1,\dots,v_n)$. We denote by $V(v_i)\subseteq V$ the subset of vertices that are connected by a path from $v_i$.

    A graph $G= (V,E)$ is connected if there are paths that contain each pair of vertices $v_i,v_j$ with $i\neq j$. A rooted tree is a graph with a distinguished 
    node $v_i$, called \textit{the root}, such that for each $v_j\in V,$ with $i\neq j$, $v_i$ is connected to $v_j$. We will note as $I(v_i)\subseteq V$ the subset of vertices that forms the rooted tree $G_{v_i} = (I(v_i),E(v_i))$ which is a subgraph of $G$.

\end{definition}

\begin{proposition}[Invariance]
    If the instruction graph $D_I$ of a virus machine $\Pi$ of degree $(p,q)$, with $p,q\geq 1$,
    $$\Pi = (H,I,D_H,D_I,G_C,n_1,\dots,n_p,i_1,h_{out}),$$
    is not a rooted tree with root $i_1$, then there exists another virus machine $\Pi'$ of degree $(p,q')$, with $q'\leq q$, which has the same computation.
\end{proposition}
\begin{proof}
    Let $\Pi$ be the virus machine fixed in the statement, setting the instruction graph to $D_I=(I,E_I,w_I)$, as it is not a rooted tree with root $i_1$; then $I(i_1)\neq I$. Let $\Pi'$ be the virus machine of degree $(p,q')$ with  $q'=|I(i_1)|)$, defined as $\Pi$ but with a new instruction graph $D_{I(i_1)}= (I(i_1),E_{I(i_1)},w_{I(i_1)})$.

    Due to the semantics associated with virus machines, any instruction that can be activated must be connected by a path from the initial instruction; thus, the set of instructions of $\Pi$ that can be activated at some instant of the computation is contained in $I(i_1)$, therefore $\Pi'$ has the same computation.
\end{proof}

Using this result, from now on, all defined virus machines are supposed to be rooted trees with root $i_1$, which is $i_1$ the initial instruction. In addition, the same notation of the components of a virus machine $\Pi$ is used for the following results. 

\subsection{State-of-the-art}
This subsection is devoted to reviewing results prior to this work on the computing power of VMs with respect to certain classes or families of computable numbers.

\begin{figure}[ht]
    \centering
    \includegraphics[width=0.6\textwidth]{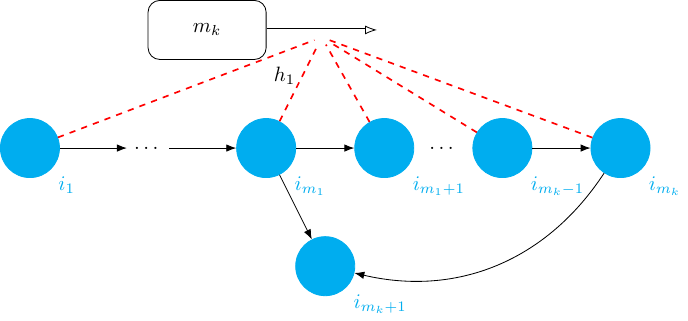}
    \caption{A virus machine generating $\mathrm{NFIN}$ for $\mathrm{NVM}(1,*,*)$.}
    \label{fig:NFIN-host}
\end{figure}

\begin{figure}[ht]
    \centering
    \includegraphics[width=0.6\textwidth]{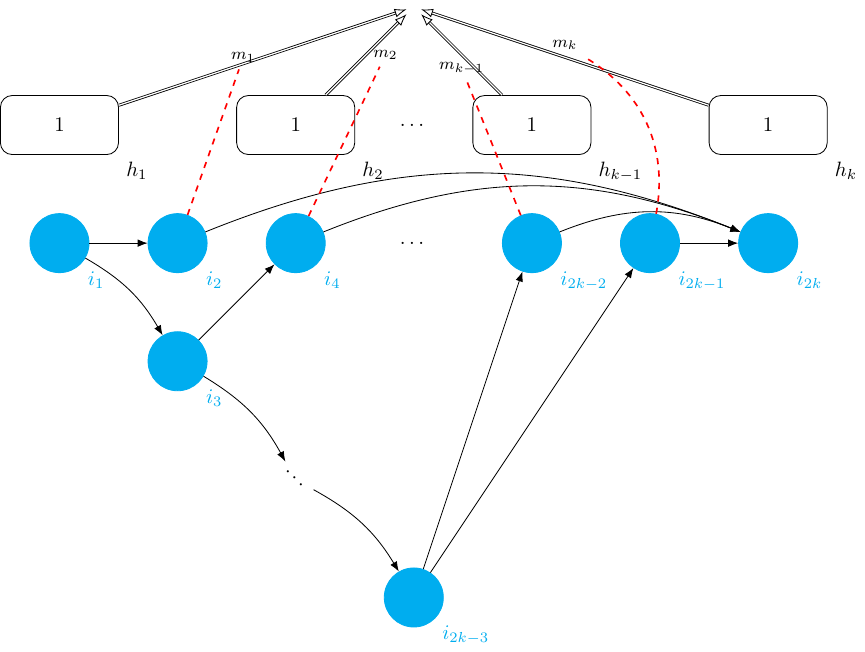}
    \caption{A virus machine generating $\mathrm{NFIN}$ for $\mathrm{NVM}(*,*,1)$.}
    \label{fig:NFIN-numhost}
\end{figure}

The state-of-the-art is presented in Table~\ref{tab:tab0}. The virus machines in the generating mode are Turing Universal; that is, they can generate recursively enumerable sets of numbers ($NRE$)~\cite{VM} for unbounded restrictions. This power is severely reduced when the last ingredient is reduced; more precisely, a semilinear set characterisation ($SLIN$) is proved for $\mathrm{NVM}(*,*,2)$~\cite{computing-VM}. From now on, not characterisations but inclusions have been proven, for finite sets ($\mathrm{NFIN}$) they are contained in $\mathrm{NVM}(1,*,*)$ and $\mathrm{NVM}(*,*,1)$~\cite{GCRVM}. Finally, the set of power of two numbers is contained in $\mathrm{NVM}(2,7,*)$~\cite{GCRVM}.

\begin{table}[ht]
    \centering
    \begin{tabular}{|c|c|c|c|c|}
        \hline
        Family of sets & Relation & Hosts & Instructions & Viruses \\\hline
        $NRE$~\cite{VM} & = & * & * & * \\\hline
        $SLIN$~\cite{computing-VM} & = & * & * & 2 \\\hline
        $\mathrm{NFIN}$~\cite{GCRVM} & $\subseteq$ & 1 & * & * \\
        & $\subseteq$ & * & * & 1 \\\hline
        $\{2^n \mid n\geq 0\}$~\cite{GCRVM} &$\subseteq$& 2 & 7 & * \\\hline
    \end{tabular}
    \caption{Previous results: Minimum resources needed to generate family subsets of natural numbers.}
    \label{tab:tab0}
\end{table}

An interesting and natural question is can we further restrict or provide better lower bounds, for known results about VMs?
That is, provide ``better'' characterisations of finite sets or even other families of sets such as the singleton sets, see, for instance, Table \ref{tab:tab0}.
As we focus on finite sets later, let us see the VMs used in~\cite{GCRVM} to generate finite sets. For $\mathrm{NVM}(1,*,*)$ the VM presented in Figure~\ref{fig:NFIN-host}, and for $\mathrm{NVM}(*,*,1)$ the Figure~\ref{fig:NFIN-numhost}. The corresponding lemmas were called (viruses) and (hosts), respectively, and we follow the same notation in this work.

\section{Novel results with old ingredients}\label{sec:oldIng}

Having set the state-of-the-art of VM in generating mode, let us see new results by the boundary of those old ingredients: the hosts, the instructions, and the viruses during the computation (see Section \ref{def:vm}). This section will start with finite sets, just to complete the previous results, and then it is organised from weak to stronger power (in the sense of bigger sets), in the end a table summarising all the novel results is presented.

\subsection{Finite sets}
\begin{lemma}[Viruses-host]\label{lem:2*2}
        Let $F = \{m_1,\dots,m_k\}$ a nonempty finite set of natural numbers with $m_i>0$. Then $F$ can be generated by a virus machine of $2$ hosts, $2k+1$ instructions, and $2$ virus in each host at most.
\end{lemma}

\begin{figure}[ht]
  \centering    
  \includegraphics[width = 0.75\linewidth]{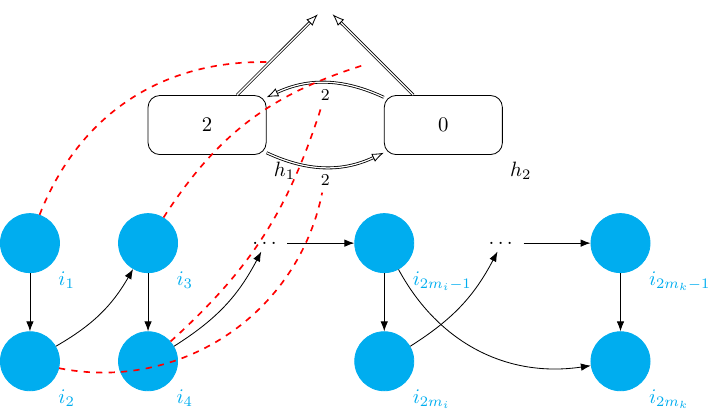}
  \caption{Virus machine generating the finite set $F=\{m_1,\dots,m_k\}$.}
  \label{fig:NFIN-tree_pn}
\end{figure}
\begin{proof}
Let $\Pi$ be the virus machine of degree $(2,2m_k)$ defined as\\
$\Pi = (\Gamma, H, I, D_{H}, D_{I}, G_{C}, n_{1}, n_2, i_{1},h_{out})$, where:

\begin{enumerate}
\item $\Gamma = \{v\}$;
\item $H = \{h_{1},h_2\}$;
\item $I = \{i_{1}, \dots , i_{2m_k}\}$;
\item
  $D_{H} = (H \cup \{h_{out}\}, E_H=\{(h_{1}, h_{2}),(h_{1}, h_{out}),(h_{2}, h_{1}),(h_{2}, h_{out})\}, w_{H}),$ where
  $w_{H}((h_{1}, h_{2})) = w_{H}((h_{2}, h_{1})) =2$ and\\
  $w_{H}((h_{1}, h_{out})) = w_{H}((h_{2}, h_{out})) = 1$;
\item
  $D_{I} = (I , E_I , w_I )$, where $E_I=\{(i_{a}, i_{a+1}) \ | \ a\in \{1,\dots, 2m_k-1\}\}\cup\\\{(i_{2m_i-1},i_{2m_k})\ | m_i\in F\}$,
  $w_{I}((i_{j}, i_{j'})) = 1,$ for each $(i_j,i_{j'})\in E_I$;
\item $G_{C} = (I \cup E_{H},E_C)$, where \\
$\begin{array}{rl}
       E_C = &\displaystyle\bigcup_{j \in \{0,\dots, m_k\}, j \text{ even }}(\{i_{2j+1},(h_1,h_{out})\},\{i_{2j},(h_1,h_{2})\})\cup\\
        & \displaystyle\bigcup_{j \in \{0,\dots, m_k\}, j \text{ odd }}(\{i_{2j+1},(h_2,h_{out})\},\{i_{2j},(h_2,h_{1})\});
\end{array}$
\item $n_{1} =2$ and $n_2= 0$;
\end{enumerate}

A visual representation of this virus machine can be found in Figure~\ref{fig:NFIN-tree_pn}. Let us prove that for each $m_i\in F$,
there exists a computation of $\Pi$ such that it produces $m_i$ viruses
in the environment in the halting configuration. Let $m_i$ be the generated number; the following invariant holds:

$$\varphi(x) \equiv \left\{\begin{array}{cc}
    \C_{2x} = (2,0,i_{2x+1},x) & x \text{ even}, \\
    \C_{2x} = (0,2,i_{2x+1},x) & x \text{ odd},
\end{array}\right.$$

for each $0\leq x\leq m_i-1$. In particular, $\varphi(m_i-1)$ is true, let us suppose that $m_i$ is odd, then the following computation is verified:
\begin{equation*}
    \begin{split}
        \C_{2(m_i-1)} &= (2,0,i_{2m_i-1}, m_i-1), \\
        \C_{2m_i} &= (1,0,i_{2m_k}, m_i), \\
        \C_{2m_i+1} &= (1,0,\#, m_i), 
    \end{split}
\end{equation*}

For $m_i$ even the computation is analogous, hence the computation halts in $2m_i+1$ steps and the number generated is $m_i$.
\end{proof}

Another interesting result is that this inclusion is strict.

\begin{proposition}\label{prop:2*2}
    $\mathrm{NFIN}\subsetneq \mathrm{NVM}(2,*,2)$.
\end{proposition}
\begin{proof}
    Inclusion is direct by the Lemma~\ref{lem:2*2}. Let us now focus on inequality; for that, we construct a virus machine from~\cite{vm_matriz_bwmc2024}, extending the work from~\cite{vm_matriz2024}, which generates the set of all natural numbers except the zero, which verifies the restrictions of the proposition. 

     Let $\Pi_{Nat} = (\Gamma, H, I, D_{H}, D_{I}, G_{C}, 1,0, i_{1}, h_{out})$, where:

\begin{enumerate}
\item $\Gamma = \{v\}$;
\item $H = \{h_{1},h_2\}$;
\item $I = \{i_{1}, \dots , i_4\}$;
\item
  $D_{H} = (H \cup \{h_{out}\}, \{(h_{1}, h_{2}),(h_{2}, h_{out}),(h_{2}, h_{1})\}, w_{H}),$ where\\
  $w_{H}((h_{1}, h_{2})) = 2$ and
  $w_{H}((h_{2}, h_{out})) = w_{H}((h_{2}, h_{1})) = 1$;
\item
  $D_{I} = (I , E_I , w_I ),$ where $E_I=\{(i_1,i_2),(i_2,_3),(i_3,i_1),(i_3,i_4)\}$, and\\
  $w_{I}((i_{j}, i_{j'})) = 1 \ \forall (i_j,i_{j'})\in E_I$;
\item $G_{C} = (I \cup E_{H},E_C)$, where $E_C =\{\{i_1,(h_1,h_2)\},\{i_2,(h_2,h_1)\},\\
\{i_3,(h_2,h_{out})\}\}$;
\item $h_{out} = h_{0}$;
\end{enumerate}

\begin{figure}[ht]
  \centering    
  \includegraphics[width = 0.5\linewidth]{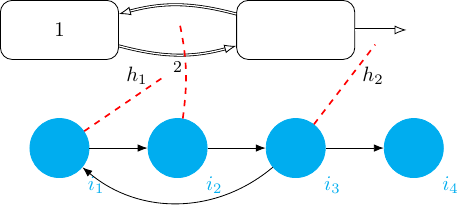}
  \caption{Virus machine generating the set of natural numbers $\N\setminus\{0\}$.}
  \label{fig:natset}
\end{figure}

A visual representation of this virus machine can be found in
Fig.~\ref{fig:natset}. Now, let us prove that for each $n\in\N$, there exists a halting computation generating the number $n$. For generating this number, the following invariant holds:

\begin{equation*}
    \begin{split}
        \varphi(k) \equiv \C_{3k} =(1,0,i_1,k), \text{ for each } 0\leq k\leq n-1
    \end{split}
\end{equation*}

In particular, $\varphi(n-1)$ is true, then the following configuration is verified $\C_{3(n-1)} = (1,0,i_1,n-1)$, from here, after the $4$ transition steps the halting configuration is reached $\C_{3n+1} =(1,0,\#,n)$, whose output is the natural number $n$.
\end{proof}

\subsection{Singleton Sets}

Now let us move to the second family of sets, the Singleton sets, these are sets of natural numbers with only one element, in this work we include the empty set in this family.

\begin{theorem}\label{teo:singlesets}
The following sets of numbers are equivalent to singleton sets:
   \begin{enumerate}
       \item $\mathrm{NVM}(1,*,1)$;
       \item $\mathrm{NVM}(*,1,*)$; 
       \item $\mathrm{NVM}(1,1,1)$
   \end{enumerate}
\end{theorem}

\begin{proof}
    The proof of equivalence is done by the double inclusion technique.
    \begin{enumerate}
        \item \boxed{\subseteq} Let $\Gamma = \{v\}$ be a singleton set of natural number $v\in \N$, then it can be generated by the VM $\Pi_{sing_1}$ of degree $(1,1)$ depicted in Figure~\ref{fig:sing-1}. Here the initial configuration is $\C_0 = (1,i_1,0)$ and in the following configuration, a virus is consumed and replicated by the weight of the arc, that is, $v$, and sent to the environment, leading to the halting configuration $\C_1 =(0,\#, v)$. Thus, after one transition step, the set generated is $\{v\}$.

    \begin{figure}[h!]
        \centering
        \includegraphics[width= 0.3\linewidth]{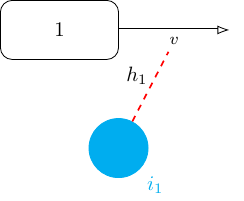}
        \caption{The VM $\Pi_{sing_1}$ generating the singleton set $\{v\}$.}
        \label{fig:sing-1}
    \end{figure}

    \boxed{\supseteq} Suppose any VM with only one host and one virus: the host can only be attached to the environment, and let us fix that the weight of that channel is $w\in \N$.
    Thus, the only number generated is $w$ or none, depending on the instruction graph (if the computation halts or not). Thus, we generate a singleton set.
    
    \item \boxed{\subseteq} Let us use the VM $\Pi_{sing_1}$ depicted in Figure~\ref{fig:sing-1} as it only has one instruction and the inclusion has already been proven.
    
    \boxed{\supseteq} With only one instruction, there are two possibilities in the instruction graph: \begin{itemize}
        \item The node with a self-arc, which creates an infinite loop, thus a non-halting computation and generating the empty set.
        \item The node with no arcs, thus the machine, halts after only one transition step as there is no other possible path. In this sense, two options can be separated:
        \begin{itemize}
            \item The instruction is attached to a channel which is attached to the environment, generating a singleton set.
            \item The instruction is not attached to a channel which is attached to the environment, thus the set generated is $\{0\}$.
        \end{itemize}
    \end{itemize}

    \item Lastly, the left side inclusion is again using $\Pi_{sing_1}$, and the right side is direct as the restrictions are stronger than the previous statements.
    \end{enumerate}
\end{proof}

\subsection{Finite linear progressions}

The computing power of virus machines highly depends on the instruction graph; to highlight this, let us see the following result when we bound by $2$ the amount of instructions. For this, we will fix the following notation: 

Let $\mathrm{NLinFIN} = \bigcup_{x\in\N}\bigcup_{n\in\N}\bigcup_{N\in\N}(\{x + n\cdot i: 0\leq i\leq N\}) \cup \{\ \emptyset\}$, be the family of finite linear progressions. The following result holds.

\begin{proposition}\label{prop:*2*}
    $\mathrm{NVM}(p,2,*) = \mathrm{NLinFIN},$ for each $p\geq 2$.
\end{proposition}

\begin{proof}
    \boxed{\supseteq} For any $x,n,N\in \N$, let us see that there exists a VM $\Pi_{Lin}$ of degree $(2,2)$ that generates the set $\{x + n\cdot i: 0\leq i\leq N\}$. The virus machine can be depicted in Figure~\ref{fig:host-inst}.

    \begin{figure}[ht]
        \centering
        \includegraphics[width= 0.4\linewidth]{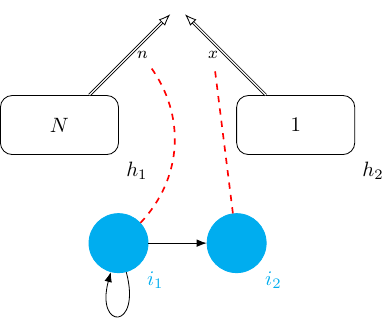}
        \caption{The VM $\Pi_{Lin}$ generating the set $\{x + n\cdot i: 0\leq i\leq N\}$.}
        \label{fig:host-inst}
    \end{figure}

    Suppose that the number generated is $x + n\cdot k$ with $0\leq k \leq N$, then the following computation holds.
    
    \begin{equation*}
        \begin{split}
            \C_{0} &= (N,1,i_1,0), \\
            \C_{1} &= (N-1,1,i_1,n), \\
            & \vdots  \\
            \C_{k-1} &= (N-(k-1),1,i_1,n\cdot(k-1)), \\
            \C_{k} &= (N-k,1,i_2,n\cdot k), \\
            \C_{k+1} &= (N-k,1-1,\#,x + n\cdot k),
        \end{split}
    \end{equation*}

    Therefore, the number is generated, the other inclusion is straightforward.

    \boxed{\subseteq} Some previous considerations should be taken into account. First, let us fix that $I=\{i_1,i_2\}$ is the set of instructions and $i_1$ is the initial instruction. We will consider the instruction graphs where there is at least one halting computation. On the other hand, we will consider those in which there is more than one computation, that is, there is a least one non-deterministic decision, otherwise only the singleton sets can be generated (which are included in $\mathrm{NLinFIN}$). 

    With all of this in mind, only one instruction graph remains, the same as depicted in Figure~\ref{fig:host-inst}. What can be different is the instruction-channel graph, if $i_1$ is attached to channel not connected to the environment, then only singleton sets can be generated, otherwise the arithmetic progression as stated before.
\end{proof}

\subsection{Finite linear combinations}
Continuing with the idea of the previous subsection, let us see that with 3 instructions we still have a strong limitation in computational power. First, we define a family of sets that we will try to characterise, let:

\begin{equation*}
    \begin{split}
        a_{w_1,w_2,N_1,N_2} =& \{w_1x+w_2y+r\ |\ 1\leq x\leq N_1,\ 1\leq y\leq N_2\};\\
        b_{w_1,w_2,N_1,N_2} =& \{f_{w_1,w_2}^{N_1,N_2}(x,y)\ |\ 1\leq x\leq \min(N_1,N_2),\\
        & \ 1\leq y\leq |N_2-N_1|\};\\
        f_{w_1,w_2}^{N_1,N_2}(x,y) =& \left\{\begin{array}{ll}
            (w_1+w_2)x+r, & x < \min(N_1,N_2); \\
            (w_1+w_2)N_1 + w_2y + r, & x=N_1\wedge N_1<N_2,\\
            (w_1+w_2)N_2 + w_1y + r, & x=N_2\wedge N_2<N_1,\\  
        \end{array}\right.
    \end{split}
\end{equation*}

For each $w_1,w_2,N_1,N_2\in\N$. Finally, letting
$A = \{a_{w_1,w_2,N_1,N_2}\}_{w_1,w_2,N_1,N_2\in\N},\\ B = \{b_{w_1,w_2,N_1,N_2}\}_{w_1,w_2,N_1,N_2\in\N},$ be two families of sets of natural numbers, we define finite linear combinations $\mathrm{NCombFIN} = \{\emptyset\}\cup A \cup B$.

\begin{proposition}\label{prop:*3*}
    $ \mathrm{NCombFIN} = \mathrm{NVM}(p,3,*),$ for each $p\geq 3$.
\end{proposition}

\begin{proof}
\boxed{\subseteq} The idea of this inclusion is to prove that there exists a virus machine for generating: (i) the empty set, (ii) the family of sets $A$, and (iii) the family of sets $B$.

\begin{itemize}
    \item[$(\emptyset)$] This is trivial, any VM with no halting computations generates the empty set, for instance, a VM of degree $(p,3)$ with $p\geq 1$, that has a loop of size $3$, thus there cannot be a halting computation and the empty set is generated.
    \item[$(A)$] Let $\Pi_1$ be the virus machine of degree $(3,3)$ visually pre\-sented in Figure~\ref{fig:nlincombfin}. Let us suppose that the number ge\-nerated is $m = w_1x'+ w_2y' +r$; then, the following compu\-ta\-tion holds. The initial configuration is $\C_0 = (N_1,N_2,1,i_1,\\0),$ from here, one virus is transmitted from host $h_1$ to the environment replicated by $w_1$, the following instruction is non-deterministically chosen between $i_1$ and $i_2$, here we choose $i_1$. This process is repeated $N_1-x'-1$ times. Thus we reach the following configuration, $\C_{x'-1} = (N_1-
    x'+1,N_2, 1, i_1,w_1(x'-1)).$ Here we choose instruction $i_2$, leading to the configuration $C_{x'}= (N_1-x',N_2,1,i_2,w_1x')$. Now the process is analogous with instruction $i_2$ and host $h_2$. After $y'$ transition steps, we choose instruction $i_3$, which will open the channel $(h_3,h_0)$, reaching the halting configuration: $C_{x'+y'+2}= (N_1-x',N_2-y', 0, \#, w_1x'+w_2y'+r),$ that is, after $x'+y'+2$ steps, the machine halts and sends $w_1x'+w_2y'+r$ viruses to the environment.

\begin{figure}[ht]
    \centering
    \includegraphics[width = 0.6\linewidth]{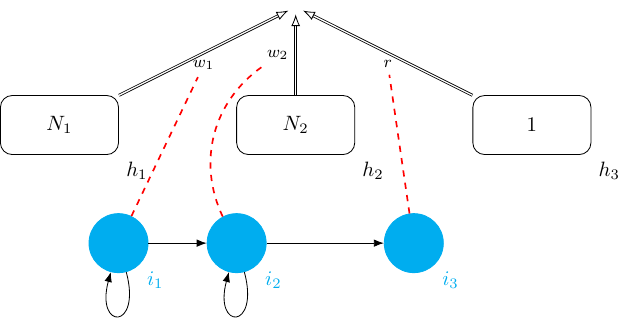}
    \caption{Virus machine $\Pi_1$.}
    \label{fig:nlincombfin}
\end{figure} 

\item[$(B)$] Let $\Pi$ be the VM of degree $(3,3)$ presented in Figure~\ref{fig:loop2}. Let us see that it can generate any $b_{w_1,w_2,N_1,N_2},$ for each $w_1,w_2,N_1,N_2\in \N.$ The main difference with the previous VM is the loop between instruction $i_1$ and $i_2$, that means that we are sending the same amount of viruses from host $h_1$ and host $h_2$ to the output region, unless the loop is repeated more than $\min(N_1,N_2)$ times, then we are sending viruses only from the host that remains some viruses. After that we send $1$ viruses from host $h_3$ and replicated by $r$. Finally, the halting configuration is reached, having sent any of the elements from $b_{w_1,w_2,N_1,N_2}$.

        \begin{figure}[ht]
            \centering
            \includegraphics[width = 0.6\linewidth]{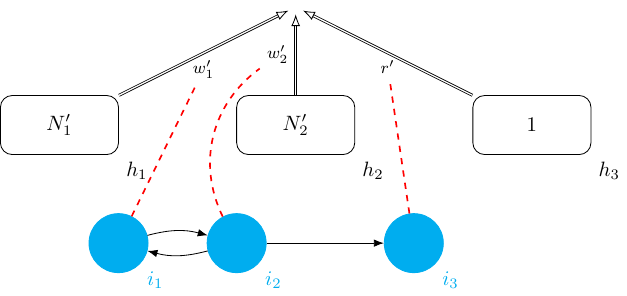}
            \caption{Virus machine with a loop of size 2 in the instruction graph.}
            \label{fig:loop2}
        \end{figure}
        
\end{itemize}

\boxed{\supseteq} The idea of this proof is based on the possible loops that can exist in a virus machine with 3 instructions. It is important to note that the way to maximise the amount of numbers that can be generated by an explicit virus machine depends on the amount (or size) of nondeterministic loops in the device. 

    We will fix that the initial configuration is always denoted by $i_1$, and there has to be at least one instruction with out-degree zero; otherwise, we are generating the empty set (which is trivially included). We will fix the instruction with out-degree zero is $i_3$.

    Lastly, it is important to note that the finite linear progression is included in this family of sets, we just need to fix that $w_2=0$.

    \begin{itemize}
        \item[Size 0] With no loops, we generate singleton sets and sets of size two, fixing that $i_1$ is attached to both $i_2$ and $i_3$ with weight $1$.
        \item[Size 1] Here we have the VM presented the previous VM presented in Figure~\ref{fig:nlincombfin}. The loops can only be in $i_1$ and $i_2$. If we fix only one, we can only generate a subset of the previous VM.

        Regarding the host graph, we will suppose that at least one of the instructions from $i_1$ and $i_2$ is not attached to a channel associated with the environment. But in those cases we are generating finite linear progressions, that is included in this family of sets. 
        
        \item[Size 2] Fixed that the out-degree of instruction $i_3$, the only possibility is that the loop of size $2$ is between instructions $i_1$ and $i_2$. In addition to this, there can exist other loops of size $1$, these are omitted because we are reaching the previous case. 

        Regarding the host graph, to generate bigger sets than the singleton sets, we need that at least one instruction of $i_1$ and $i_2$ send viruses to the environment. We will suppose both of them send viruses from different channels as in other case we are generating the finite linear progressions that are included.

        With all previous cases discarded, there is only one kind of virus machine remaining, this is represented in Figure~\ref{fig:loop2}.
        
        \item[Size 3] This is trivial, as there is no halting computation, thus the empty set is generated.
    \end{itemize}
\end{proof}

\subsection{Discussion of old ingredients}\label{subsec:discOld}

In this subsection a brief discussion of the results obtained with the previous results is presented in Table~\ref{tab:1}. Here some new characterisations have been proven below the finite sets, such as the singleton sets, the $\mathrm{NLinFIN}$ and the $\mathrm{NCombFIN}$. This was possible due to the limit of the number of instructions. Lastly, a new inclusion of the finite sets (and even more, strict inclusion) has been stated.

\begin{table}[ht]
\centering
\begin{tabular}{|l|c|l|l|l|}
\hline
Family of sets               & Relation     & host  & inst. & n.v.  \\ \hline
\multirow{1}{*}{Singleton}   & (Theorem~\ref{teo:singlesets}) $=$ &  1   &  1  &  1   \\ \hline
\multirow{1}{*}{$\mathrm{NLinFIN}$}   & (Proposition~\ref{prop:*2*}) $=$ & 2 & 2 & * \\ \hline
\multirow{1}{*}{$\mathrm{NCombFIN}$}  & (Proposition~\ref{prop:*3*}) $=$ & 3 & 3 & * \\ \hline
\multirow{3}{*}{Finite sets} & \cite{GCRVM} $\subseteq$  & 1  & * & * \\  
                             & \cite{GCRVM} $\subseteq$ & * & * & 1  \\ 
                             & (Proposition~\ref{prop:2*2}) $\subsetneq$ & 2 & *  & 2    \\ \hline
\multirow{1}{*}{$\mathrm{SLIN}$}       & \cite{computing-VM}  $=$ & * & *  & 2   \\ \hline 
$\mathrm{NRE}$                         & \cite{computing-VM} $=$ & *  & * & *  \\ \hline
\end{tabular}
\caption{Discussion of minimum resources needed for generating some family of sets of natural numbers.}
\label{tab:1}
\end{table}

\section{Novel results with new ingredients}\label{sec:newIng}

For virus machines in generating mode, the notation used in previous works \cite{GCRVM} was $\mathrm{NVM}(p,q,r)$ where it denotes the family of natural number sets generated by virus machines with at most $p$ hosts, $q$ instructions and $r$ viruses at each instant in the computation. The computational completeness has been proven when these ingredients are unbounded \cite{VM}, that is, $\mathrm{NRE} = \mathrm{NVM}(*,*,*)$, but the power decreases when one of them is bounded, for example, a characterisation of semilinear sets has been proven in \cite{computing-VM}, that is, $\mathrm{SLIN}= \mathrm{NVM}(*,*,r)$ for each $r\geq 2$. However, the three ingredients mentioned seem to provide poor information about the computation. 
Thus, we propose the following notation:

\begin{equation}
    \mathrm{NVM}_{\beta}(h_p,i_q,nvh_r,wc_s,outd_t,\alpha_{\hbar}^{u},\alpha_{\iota}^{v}),
\end{equation}

where:
\begin{itemize}
    \item $p,q,r\geq 1$ represent the same as before,
    \item $\beta\in\{T,F\}$ represents if each channel is attached to only one instruction, that is, if there is a bijection between instructions and channels, then $\beta = T$, otherwise we have $\beta = F$,
    \item  $s\geq 1$ is the maximum weight of the arcs in the host graph,
    \item $t$ is the maximum $out-degree$ of each host in the host graph,
    \item $u\geq 0$ is the greatest loop in the host graph.
    \item $v\geq 0$ is the greatest loop in the instruction graph.
\end{itemize}

The rationale for why we chose these new ingredients will be clearly shown in the following sections; however, let us take a brief look at an introductory idea behind each new ingredient.

First, we believe that the instruction-channel graphs have something to say in the sense of computational power, that is why we fix the bijection that will show a new frontier of computational power when we combine it with other ingredients, for instance, in theorems~\ref{teo:hostGraph},~\ref{teo:InstGraph} and ~\ref{teo:SLIN**2}.

Secondly, fixing the directed graphs as trees (when $u$ or $v$ are zero) is a first intention to approach the topology structure of the graphs, as they allow us characterisations of finite sets (such as Corollary~\ref{cor:TreeHG} or Theorem~\ref{teo:hostGraph}). Nevertheless, we believe that not only whether or not there is a loop is important, but also the size of the loop. The relation between these sizes and the power of the devices will have interesting results, for instance, Theorems 
~\ref{teo:InstGraph},~\ref{teo:SLIN**2},~\ref{teo:SLIN2*2}, and~\ref{teo:NRE}, and Lemma~\ref{lem:2*2}. Note that $u=1$ is the same as $u=0$ as there are no self-arcs in the host graph.

Thirdly, the weight of the arcs is the unique way to increase the amount of viruses, fixing this weight to one makes the power falls substantially (only finite sets as seen in Theorem~\ref{teo:NFINhost}), but if it is $2$ we can get novel universality results (see Theorem~\ref{teo:NRE}).

Lastly, the out-degree of the hosts, we believe that it is also crucial in the restrictions of the computing power, it will be shown in combination with other ingredients (see all the mentioned theorems and lemmas).

\subsection{Finite sets}

\begin{proposition}\label{prop:FINHostGraph}
    If a VM generates an infinite set of natural numbers, then its host graph has at least a cycle.
\end{proposition}
\begin{proof}
    Let $\Pi$ a VM of degree $(p,q)$ whose host graph is acyclic and where $n_1,\dots, n_p\in\N$ is the initial number of hosts. Let us see that the greatest number that can be generated is bounded.

    First, remark that as there are no cycles, the host graph is a tree, to generate the greatest number, we will always choose to transmit the viruses through the maximum weight channel. This idea is shown in Figure~\ref{fig:schTreeHost}.

    \begin{figure}[ht]
        \centering
        \includegraphics[width = 0.75\linewidth]{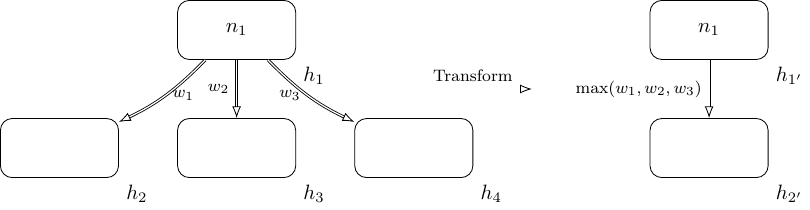}
        \caption{Scheme of the reduction in the host graph.}\label{fig:schTreeHost}
    \end{figure}

    Note that in order to obtain more replication with the fixed amount of hosts, that is, to maximise the depth, the out-degree of each host will be fixed to $1$. Thus, the host graph will have the structure shown in Figure~\ref{ffig:finalHost}.

    \begin{figure}[ht]
        \centering
        \includegraphics[width = 0.75\linewidth]{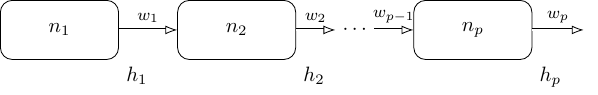}
        \caption{The host graph of $\Pi$.}\label{ffig:finalHost}
    \end{figure}

    Thus, the greatest number of viruses that can be generated is sending all viruses from host $h_1$ to host $h_2$, and then all of those viruses to host $h_3$ and so on until $h_{out}$. The number of viruses that reach the output region is:
    $$\sum_{i=1}^{p}((\prod_{j=i}^{p}w_j)\cdot n_i).$$

    In summary, the greatest number that can be generated is bounded, and thus the set generated is finite.
\end{proof}

\begin{corollary}\label{cor:TreeHG}
    $\mathrm{NVM}_F(h_*,i_*, nvh_*, wc_*,outd_*,\alpha^{\hbar}_*,\alpha^{\iota}_0) \subseteq \mathrm{NFIN}.$
\end{corollary}

With finite sets, there are some interesting results; you get the characterisation as $wc_1$ limits to a finite number of viruses.

\begin{theorem}\label{teo:hostGraph}
    The following family of sets are equal to $\mathrm{NFIN}$:
    \begin{enumerate}
        \item\label{teo:NFINhost} $\mathrm{NVM}_{F}(h_1,i_*,nvh_*,wc_s,outd_t,\alpha^{\hbar}_{0},\alpha^{\iota}_{0}),$ for each $s,t\geq 1$.
        \item\label{teo:NFINnumhost} $\mathrm{NVM}_{T}(h_*,i_*,nvh_r,wc_*,outd_t,\alpha^{\hbar}_{0},\alpha^{\iota}_{0}),$ for each $r,t\geq 1$.
    \end{enumerate}
\end{theorem}

\begin{proof}
    Separating by the two cases, we have:
    \begin{enumerate}
        \item The proof is followed by the previous corollary and the inclusions of~\cite{GCRVM}. In a simple glance, we can see in Figure~\ref{fig:NFIN-host} that both the instruction and the host graphs are trees. 
        \item On the other hand, in Figure~\ref{fig:NFIN-numhost} we can highlight that we have one channel associated with only one instruction, that is, $\beta = T$.
    \end{enumerate}
\end{proof}

Now we move to similar results with the instruction graph fixed as a tree. For this, the following result holds:

\begin{proposition}\label{prop_tree}
    If the instruction graph is a tree, then all computations halt. In addition, the number of transition steps is bounded by the depth of the tree.
\end{proposition}

\begin{corollary}\label{coro:treeIG}
    $\mathrm{NVM}_F(h_*,i_*, nvh_*, wc_*,outd_*,\alpha^{\hbar}_0,\alpha^{\iota}_*) \subseteq \mathrm{NFIN}$
\end{corollary}
\begin{proof}
    The inclusion is direct by Proposition~\ref{prop_tree}, as all computations halt, then every virus machine halts in a finite number of steps, thus the set of numbers that can be generated is finite.
\end{proof}

\begin{theorem}\label{teo:InstGraph}
    The following family of sets are equal to $\mathrm{NFIN}$:
    \begin{enumerate}
        \item $\mathrm{NVM}_{F}(h_1,i_*,nvh_*,wc_s,outd_t,\alpha^{\hbar}_{0},\alpha^{\iota}_{0}),$ for each $s,t\geq 1$.
        \item $\mathrm{NVM}_{T}(h_*,i_*,nvh_r,wc_*,outd_t,\alpha^{\hbar}_{0},\alpha^{\iota}_{0}),$ for each $r,t\geq 1$.
        \item $\mathrm{NVM}_{F}(h_p,i_*,nvh_r,wc_s,outd_t,\alpha^{\hbar}_{u},\alpha^{\iota}_{0}),$ for each $p,r,s,t,\\u\geq 2$.
    \end{enumerate}
\end{theorem}

\begin{proof}
    The proof of the first two statements is analogous to the proof of Theorem~\ref{teo:hostGraph}. For the third statement, it can be proved by applying the Lemma~\ref{lem:2*2} and the Corollary~\ref{coro:treeIG}.
\end{proof}

It is interesting to note that now with $v = 0$, we have gone from strict inclusion in Proposition~\ref{prop:2*2} to characterisation in the last statement in Theorem~\ref{teo:InstGraph}.

\subsection{Semilinear sets}

The authors in~\cite{computing-VM} characterise semilinear sets by VMs in generating mode, that is, $SLIN = (*,*,r)$, for all $r\geq 2$. Reviewing how the proof was constructed, we can assure the following result with new ingredients:

\begin{theorem}\label{teo:SLIN**2}
    $\mathrm{SLIN} = \mathrm{NVM}_{T}(h_*,i_*, nvh_r,wc_*,outd_t,\alpha^{\hbar}_{u},\alpha^{\iota}_{v})$\\ for all $r,t,u\geq2$ and $v\geq 3$.
\end{theorem}

The question that arises is can we get another trade-off in the unbounded ingredients? In this section, we will prove that we can by the demonstration of the following theorem:

\begin{theorem}\label{teo:SLIN2*2}
    $\mathrm{SLIN} = \mathrm{NVM}_{F}(h_p,i_*, nvh_r,wc_s,outd_t,\alpha^{\hbar}_{u},\alpha^{\iota}_{*})$\\ for each $p,r,s,t,u\geq2$.
\end{theorem}

\begin{proof}
    \boxed{\supseteq} This part of the proof is equal to the technique applied by the authors in~\cite{computing-VM}. That is, simulating the right-linear grammar, this is possible due to the amount of viruses of each host at each moment of the computation being bounded, and thus the number of possible configurations is finite. We still are in the same conditions, so the same proof can be made.
    
    \boxed{\subseteq} For simplicity, this part of the proof has been divided into two lemmas. Applying Lemma~\ref{lem:arit}, the arithmetic progressions are generated by this family of VMs. Lastly, Lemma~\ref{lem:union} shows the closure under union. Thus, the inclusion is formally proved.
    
\end{proof}

\begin{lemma}[Arithmetic progression]\label{lem:arit}
    For each $n,r\geq 1,$ we have the following inclusion $\{n\cdot i + r\ |\ i\geq 1\}\in \mathrm{NVM}_{F}(h_2,i_{3(n+r)}, nvh_2,
    wc_2,outd_2,\alpha^{\hbar}_{2},\alpha^{\iota}_{*})$. More precisely, it will be generated by the virus machine $\Pi_{arith}$ of degree $(2,3n+3r)$ defined as:

    \begin{equation*}
        \Pi_{arith} = (\Gamma, H = \{h_1,h_2\}, I , D_H, D_I, G_C, 0,1,i_1,h_0),
    \end{equation*}

    where,
    \begin{itemize}
        \item $I = \{i_1,\dots,i_{3n+3r}\};$
        \item $D_H = (H\cup \{h_0\}, E_H = \{(h_1,h_2),(h_1,h_0),(h_2,h_1)\}, w_H),$ where\\
        $w_H(h_1,h_0) = w_H(h_2,h_1) = 1,$ and $w_H(h_1,h_2) = 2$;
        \item $D_I = (I, E_I, w_I)$, where $\{(i_k,i_{k+1})\ |\ k\in \{1,\dots, 3n+3r-1\}\}\cup \{(i_{3n},i_1)\},$ and $w_I (i_k,i_{k'}) = 1$ for each $(i_k,i_{k'})\in E_I$;
        \item $G_C = (E_H \cup I, E_C),$ where $\{\{i_j,f(j)\}\ |\ j \in\{1,\dots, 3n+3r\}\},$ being
        $$f(j) = \left\{\begin{array}{ll}
            (h_1,h_0), & j \equiv 0 \mod 3, \\
            (h_1,h_2), & j \equiv 1 \mod 3, \\
            (h_2,h_1), & j \equiv 2 \mod 3.
        \end{array}\right.$$
    \end{itemize}  

    \begin{figure}[ht]
        \centering
        \includegraphics[width = 0.6\linewidth]{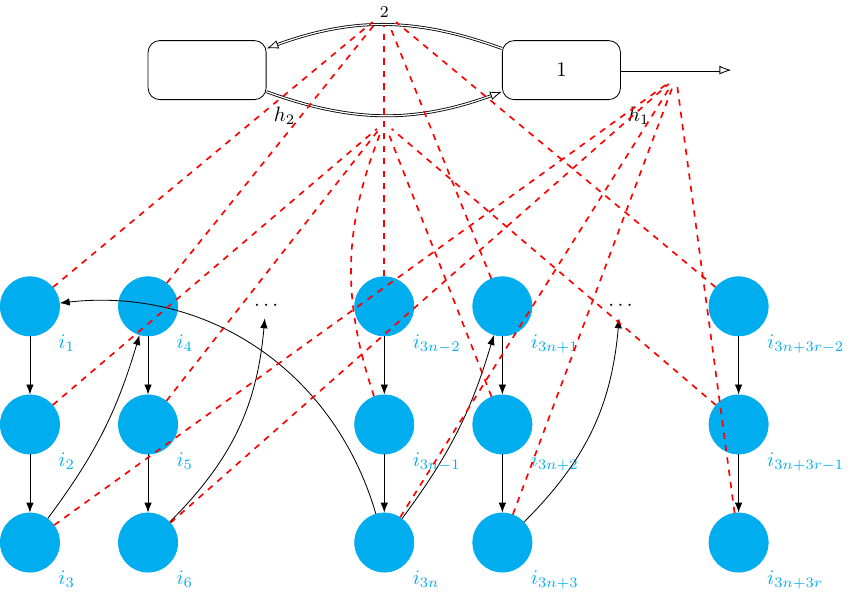}
        \caption{Virus machine $\Pi_{arith}$.}
        \label{fig:VMarith}
    \end{figure}
\end{lemma}

\begin{proof}
    Suppose that the number generated is $m\cdot n+r$ with $m\geq 1$, then the following invariant holds:

    \begin{equation*}
        \varphi(k) \equiv C_{3k} = (1,0,i_1,k\cdot n),\ \text{ for each } 0\leq k < m 
    \end{equation*}

    The idea of this invariant is taking the non-deterministic decision at instruction $i_{3n}$, that is, going back to instruction $i_1$. In particular, $\varphi(m-1)$ is true. From this, a instruction $i_{3n}$ is reached, and we take the decision to go to instruction $i_{3n+1}$ with the following configuration $C_{3m} = (1,0,i_{3n+1},m\cdot n)$. After that, it is straightforward to reach the halting configuration.

    \begin{equation*}
        C_{3m+3r} = (1,0,\#,m\cdot n + r)
    \end{equation*}
    
    Thus, there exists a computation that generates the number $m\cdot n +r$. 

    Regarding the other inclusion, that is, for each halting computation of the machine, the number generated is in the set proposed. It can be seen in a simple glance that for any generated number we will have the previous invariant formula, and after that, we send $r$ viruses to the environment. Thus, the other inclusion is proved.
\end{proof}

\begin{lemma}[Union closure]\label{lem:union}
    Let $Q_1,Q_2,\dots,Q_m\subseteq \N$ with $m>0$ arithmetical progressions, then $\cup_{j=1}^{m}Q_j \in \mathrm{NVM}_{F}(h_p,i_{*}, nvh_r,wc_s,outd_t,\alpha^{\hbar}_{u},\alpha^{\iota}_{*})$ for each $p,r,s,t,u\geq2$.
\end{lemma}

\begin{proof}
    The main idea of the proof is that we can also keep the host graph in the union of each subset. The induction technique will be used for this purpose.

    \boxed{Q_1\cup Q_2} As both $Q_1,Q_2$ are arithmetic progression, we can apply the Lemma~\ref{lem:arit}, we will denote that $Q_j$ is generated by $\Pi_j$ for each $j\in\{1,2\}$, this sub-index notation is extended to the heterogeneous networks of each machine. Note that the host graphs and the initial amount of viruses of each virus machine remains equal. The construction is visually explained in Figure \ref{fig:union1}. From which it can be seen in a simple glance that it generates $Q_1\cup Q_2$, proving that $Q_1\cup Q_2$ is in the families of the statement. 

    \boxed{\cup_{j=1}^{m-1}Q_j \cup Q_m} For the inductive step, let us suppose that a virus machine $\Pi_{m-1}$ generates the union with host graph mentioned in the base case. Applying again Lemma \ref{lem:arit}, let $\Pi_m$ the virus machine that generates the set $Q_m$. The construction of the virus machine that generates the union of both sets will be analogous to the previous one, visually explained in Figure \ref{fig:union2}, thus it has been proved that $\cup_{j=1}^m Q_j$ is in the families of the statement.
    
    \begin{figure}[ht]
        \centering
        \includegraphics[width = 0.35\linewidth]{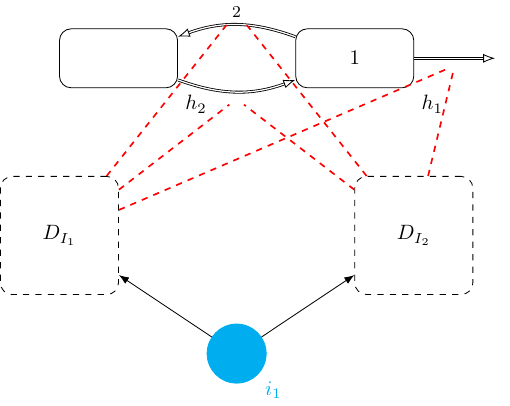}
        \caption{Visual idea of the union for generating $U_1\cup U_2$.}
        \label{fig:union1}
    \end{figure}
    
    \begin{figure}[ht]
        \centering
        \includegraphics[width = 0.35\linewidth]{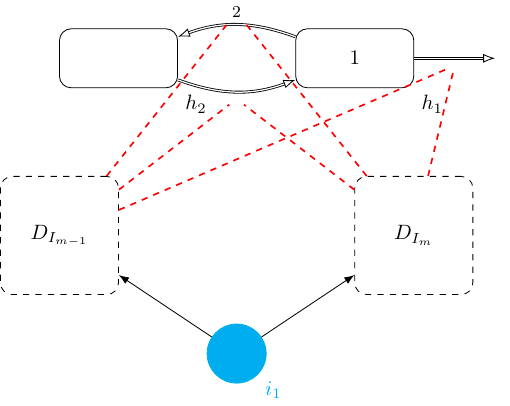}
        \caption{Visual idea of the union for generating $\cup_{j=1}^{m-1}Q_j\cup Q_m$.}
        \label{fig:union2}
    \end{figure}
\end{proof}

\subsection{Universality}

Lastly, it is interesting to note that new frontiers of universality can be obtained with these novel ingredients. 

\begin{theorem}\label{teo:NRE}
    $\mathrm{NRE} = \mathrm{NVM}_{F}(h_*,i_*,nvh_*,wc_s,out_*,\alpha^{\hbar}_*,\alpha^{\iota}_*)$ for each $s\geq 2$.
\end{theorem}

\begin{proof}
    For the proof, we refer to~\cite{computing-VM}, where the Turing completeness was proven by simulating register machines in a modular way. In that simulation, one can see on a simple glance that the weight of the channels are bounded by $2$, thus the theorem is proved.
\end{proof}

\subsection{Discussion of new ingredients}

Here, table~\ref{tab:2} summarises the previous results in addition to the new results obtained in this work with the proposed new ingredients. First, the singleton sets were characterised in the previous section with the old ingredients; note that including the new ingredients is straightforward. With finite sets, it is interesting to note that the new characterisations were previously inclusions, just fixing the host graph and/or the instruction graph to a tree. Showing the strict inclusion with the loops of size two in the host graph and size three in the instruciton graph. These loop restrictions are the same for characterising the semilinear sets, but unbounding the number of hosts instead. Note that despite the $\beta$ ingredient seems to be very restrictive, we can characterise this family of sets. The most interesting result of this section arises when we unbound the size of the loops in the instruction graph and $\beta = F$, which decreases the number of necessary hosts to two. Last but not least, the universality result from~\cite{computing-VM} has been revised with the new ingredients, showing that there is a huge frontier in computational power, from finite sets with weight $1$, to $\mathrm{NRE}$ with weight $2$.

\begin{table*}[ht]
\centering
\begin{tabular}{|l|c|l|l|l|l|c|l|l|l|}
\hline
Family of sets               & Relation & $h$ & $i$ & $nv$ & $wc$ & $outd$ & $\alpha^{\hbar}$ & $\alpha^{\iota}$ & $\beta$ \\\hline
\multirow{1}{*}{Singleton}   & $=$ (Theorem~\ref{teo:singlesets}) &  1  &  1  &  1   & 1 &  1 & 0 & 0 & T \\ \hline 
\multirow{4}{*}{Finite sets} & $=$ (Theorem~\ref{teo:NFINhost}) & 1   & *   &  *   & 1 &  1 & 0 & 0 & F \\ 
                             & $=$ (Theorem~\ref{teo:NFINnumhost}) & *   & *   &  1   & * &  1 & 0 & 0 & T \\ 
                             & $=$ (Theorem~\ref{teo:InstGraph}) & 2   & *   &  2   & 2    & 2   & 2 & 0 & F \\ 
                             &$\subsetneq$ (Lemma~\ref{lem:2*2})   & 2 & *   &  2   & 2  &  2  & 2 & 3 & F \\ \hline
\multirow{2}{*}{$\mathrm{SLIN}$}      &   $=$ (Theorem~\ref{teo:SLIN**2}) & *   & *   &  2   & *    & 2     & 2 & 3 & T \\ 
                             &   $=$  (Theorem~\ref{teo:SLIN2*2}) & 2   & *   &  2   & 2    & 2    & 2 & * & F \\ \hline
$\mathrm{NRE}$                        &   $=$ (Theorem~\ref{teo:NRE}) & *   & *   &  *   & 2    &   *       & * & * & F \\ \hline
\end{tabular}
\caption{Discussion of minimum resources needed for generating some family of sets of natural numbers.}
\label{tab:2}
\end{table*}

\section{Conclusions}\label{sec:conc}

Some directions for future work include the extension to parallel VMs e.g. \cite{vm_workflow2024,Antonio2024a}.
In such VMs more than one instruction can be active at each computation step, so some of the restrictions in the present work may not apply.
Another natural extension of the present work is to apply them to VMs for accepting inputs or as {\it transducers}, that is, computing functions using both input and output~\cite{GCRVM}.

The results in normal forms from the present work can be\-tter inform the design of applications or implementations of VMs.
For instance, perhaps with ideas from \cite{vm_matriz_bwmc2024,vm_matriz2024}, sequels of the simulator from \cite{vm_sim}, or applications of VMs such as edge detection~\cite{Antonio2024b}, or cryptography~\cite{CryptoVM} can be improved by applying simplifications based on normal forms.
Similarly, knowing which types of instruction, host, or channel graphs are used, the {\it reachability} of one configuration from another and other properties, can be decidable: in this case, searching for efficient algorithms and applications to formal verification are open problems.
The results of the classes below $NRE$ help inform the computational complexity of problem solving with a model.
In the case of VMs, a future direction is to investigate deeper into ``sub-$NRE$'' classes: in this way a better view of efficiency or lack of it, can be given for VMs.
Lastly, new and optimal lower bounds are expected to improve the results of the present work, summarised in Table \ref{tab:2}.
New normal forms can perhaps include new ingredients or semantics not previously considered.
For instance, a deeper focus on deterministic versus nondeterministic computations, as well classes and trade-offs for time and space.

\bibliographystyle{fundam}
\bibliography{ref}

@article{tree_or_not_bwmc2024,
    author = {Ram\'irez-de-Arellano, Antonio and Cabarle, Francis George C. and Orellana-Mart\'in, David  and Adorna, Henry N. and  P\'erez-Jim\'enez, Mario J.},
    title = {Virus Machines: To tree or not to tree},
pages={93--105},
    journal="20th Brainstorming Week on Membrane Computing and First Workshop on Virus Machines, 24--26 January 2024, Sevilla, Spain",
    year = "2024",
}

@article{vm_matriz_bwmc2024,
    author = {Ram{\'i}rez-de-Arellano, Antonio and Cabarle, Francis George C. and Orellana-Mart{\'i}n, David and P{\'e}rez-Jim{\'e}nez, Mario J. and Adorna, Henry N.},
    title = {Virus Machines And Their Matrix Representations},
pages={79--91},
    journal="20th Brainstorming Week on Membrane Computing and First Workshop on Virus Machines, 24--26 January 2024, Sevilla, Spain ",
    year = "2024",
}

@inproceedings{VM,
  title={Basic virus machines},
  author={Valencia-Cabrera, Luis and P{\'e}rez-Jim{\'e}nez, Mario J and Chen, Xu and Wang, Beizhan and Zeng, Xiangxiang},
  booktitle={16th International Conference on Membrane Computing (CMC16)},
  address={Valencia, Spain},
  pages={323--342},
  year={2015}
}

@article{computing-VM,
    author = {Chen, Xu and P\'erez-Jim\'enez, Mario J. and
  Valencia-Cabrera, Luis and Wang, Beizhan and Zeng, Xiangxiang},
    title = {Computing with viruses},
    journal = {Theoretical Computer Science},
    year = {2016},
    pages = {146--159},
    volume = {623},  
    doi = {10.1016/j.tcs.2015.12.006}
}

@article{GCRVM,
author = {Ram{\'i}rez{-}de{-}Arellano, Antonio and Orellana-Mart{\'i}n, David and P{\'e}rez-Jim{\'e}nez, Mario J},
year = {2023},
month = {07},
pages = {114077},
volume = {972},
title = {Generating, computing and recognizing with virus machines},
journal = {Theoretical Computer Science},
doi = {10.1016/j.tcs.2023.114077}
}

@article{chomsky_formgram1959,
  title={On certain formal properties of grammars},
  author={Chomsky, Noam},
  journal={Information and control},
  volume={2},
  number={2},
  pages={137--167},
  year={1959},
  publisher={Elsevier},
  doi = {10.1016/S0019-9958(59)90362-6}
}

@article{ivan_snpvarnorm2022,
  title={Normal forms for spiking neural {P} systems and some of its variants},
  author={Macababayao, Ivan Cedric H and Cabarle, Francis George C and de la Cruz, Ren Tristan A and Zeng, Xiangxiang},
  journal={Information Sciences},
  volume={595},
  pages={344--363},
  year={2022},
  publisher={Elsevier},
  doi= {10.1016/j.ins.2022.03.002}
}

@article{francis_acmc2022jour,
author = {Cabarle, Francis George C},
title = {{\textsc Thinking About Spiking Neural {P} Systems: 
           Some Theories, Tools, and Research Topics}},
journal={Journal of Membrane Computing},
volume = {6},
number={2},
pages={1--20},
year = {2024},
doi = {10.1007/s41965-024-00147-y}
}

@book{unconcomp_handb2021,
  title={Handbook Of Unconventional Computing (In 2 Volumes)},
  author={Adamatzky, Andrew},
  year={2021},
  publisher={World Scientific},
  address = {Singapure},
  doi={10.1142/12232}
}

@book{naco_handbook2012,
  title={Handbook of natural computing},
  author={B{\"a}ck, Thomas and Kok, Joost N and Rozenberg, G},
  year={2012},
  publisher={Springer, Heidelberg},
  doi = {10.1007/978-3-540-92910-9}
}

@article{ibarra_snpnorm2007,
  title={Normal forms for spiking neural P systems},
  author={Ibarra, Oscar H and P{\u{a}}un, Andrei and P{\u{a}}un, Gheorghe and Rodr{\'i}guez-Pat{\'o}n, Alfonso and Sos{\'i}k, Petr and Woodworth, Sara},
  journal={Theoretical Computer Science},
  volume={372},
  number={2-3},
  pages={196--217},
  year={2007},
  publisher={Elsevier},
  doi = {10.1016/j.tcs.2006.11.025}
}

@inproceedings{vm_matriz2024,
author={Ram{\'i}rez-de-Arellano, Antonio
and Cabarle, Francis George C.
and Orellana-Mart{\'i}n, David
and P{\'e}rez-Jim{\'e}nez, Mario J.
and Adorna, Henry N.},
editor="Ferr{\'a}ndez Vicente, Jos{\'e} Manuel
and Val Calvo, Mikel
and Adeli, Hojjat",
title="Matrix Representation of Virus Machines",
booktitle="Bioinspired Systems for Translational Applications: From Robotics to Social Engineering",
year="2024",
publisher="Springer Nature Switzerland",
address="Cham",
pages="420--429",
isbn={978-3-031-61137-7},
doi={10.1007/978-3-031-61137-7_39},
}

@article{ginsspan1966,
  title={Semigroups, Presburger formulas, and languages},
  author={Ginsburg, Seymour and Spanier, Edwin},
  journal={Pacific journal of Mathematics},
  volume={16},
  number={2},
  pages={285--296},
  year={1966},
  publisher={Mathematical Sciences Publishers},
  doi={10.2140/PJM.1966.16.285}
}

@article{vm_sim,
  title={Simulating and validating virus machines},
  author={Orellana-Mart{\'i}n, David and Ram{\'i}rez-de-Arellano, Antonio and P{\'e}rez-Jim{\'e}nez, Mario J},
  year={2025},
  journal={Natural Computing},
  doi={10.1007/s11047-025-10013-0}  
}

@inproceedings{brownca2010,
  title={Efficient computation in brownian cellular automata},
  author={Lee, Jia and Peper, Ferdinand},
  booktitle={Natural Computing: 4th International Workshop on Natural Computing Himeji, Japan, September 2009 Proceedings},
  pages={72--81},
  year={2010},
  organization={Springer},
  doi = {10.1007/978-4-431-53868-4_8}  
}

@inproceedings{peper_normbca2012,
  title={Simplifying brownian cellular automata: two states and an average of two rules per cell},
  author={Peper, Ferdinand},
  booktitle={2012 Third International Conference on Networking and Computing},
  pages={367--370},
  year={2012},
  organization={IEEE}
}

@inproceedings{nipkow2008linear,
  title={Linear quantifier elimination},
  author={Nipkow, Tobias},
  booktitle={International Joint Conference on Automated Reasoning},
  pages={18--33},
  year={2008},
  organization={Springer}
}

@article{Antonio2024b,
author = {Ramirez-de-Arellano, Antonio and Orellana-Martin, David and Perez-Jimenez, Mario J.},
title = {Bridges Between Spiking Neural Membrane Systems and Virus Machines},
journal = {International Journal of Neural Systems},
volume = {31},
number = {06},
pages = {2450034},
year = {2024},
doi = {10.1142/S0129065724500345}
}

@article{Antonio2024a,
  title={Parallel virus machines},
  author={Ram{\'i}rez-de-Arellano, Antonio and Orellana-Mart{\'i}n, David and P{\'e}rez-Jim{\'e}nez, Mario J},
  journal={Journal of Membrane Computing},
  pages={1--11},
  year={2024},
  publisher={Springer},
  doi={10.1007/s41965-024-00160-1},
}

@incollection{vm_workflow2024,
author="Ram{\'i}rez-de-Arellano, Antonio
and Cabarle, Francis George C.
and Orellana-Mart{\'i}n, David
and Riscos-N{\'u}{\~{n}}ez, Agust{\'i}n
and P{\'e}rez-Jim{\'e}nez, Mario J.",
editor="Juan, Angel A.
and Faulin, Javier
and Lopez-Lopez, David",
title="Virus Machines at Work: Computations of Workflow Patterns",
booktitle="Decision Sciences",
year="2025",
publisher="Springer Nature Switzerland",
address="Cham",
pages="208--219",
isbn="978-3-031-78238-1"
}

@article{CryptoVM,
author = {P{\'e}rez-Jim{\'e}nez, Mario J and Ram{\'i}rez{-}de-Arellano, Antonio and Orellana-Mart{\'i}n, David},
year = {2023},
month = {12},
volume = {13},
number = {1},
pages = {21831},
title = {Attacking cryptosystems by means of virus machines},
journal = {Scientific Reports},
doi = {10.1038/s41598-023-49297-6}
}

\end{document}